\documentclass{article}

\usepackage{microtype}
\usepackage{graphicx}
\usepackage{subfigure}
\usepackage{booktabs} 
\usepackage{xfrac}  
\usepackage{hyperref}

\usepackage{fancyhdr}
\usepackage{xcolor} 
\usepackage{algorithm}
\usepackage{algorithmic}
\usepackage{natbib}
\usepackage{eso-pic} 
\usepackage{forloop}
\usepackage{url}

\usepackage{settings}
\usepackage{authblk}

\author[1]{Alberto Gonz\'alez-Sanz}
\author[1]{Lucas De Lara}
\author[2]{Louis Béthune}
\author[1]{Jean-Michel Loubes}
\affil[1]{Institut de Mathématiques de Toulouse, Université Paul Sabatier}
\affil[2]{Institut de Recherche en Informatique de Toulouse, Université Paul Sabatier}

\title{GAN Estimation of Lipschitz Optimal Transport Maps}
\date{}

\begin{document}

\maketitle

\begin{abstract}
This paper introduces the first statistically consistent estimator of the optimal transport map between two probability distributions, based on neural networks. Building on theoretical and practical advances in the field of Lipschitz neural networks, we define a Lipschitz-constrained generative adversarial network penalized by the quadratic transportation cost. Then, we demonstrate that, under regularity assumptions, the obtained generator converges uniformly to the optimal transport map as the sample size increases to infinity. Furthermore, we show through a number of numerical experiments that the learnt mapping has promising performances. In contrast to previous work tackling either statistical guarantees or practicality, we provide an expressive and feasible estimator which paves way for optimal transport applications where the asymptotic behaviour must be certified.
\end{abstract}


\section{Introduction}
\label{intro}

An \emph{optimal transport map} is the fundamental object of Monge's seminal formulation of optimal transport \citep{monge1781memoire}. It transforms one distribution into another with minimal effort. Formally, given two probability distributions $P$ and $Q$ on $\Omega \subseteq \R^d$, an optimal transport map from $P$ to $Q$ is a solution to,
\begin{equation}\label{monge}
    \min_{T \in \mathcal{T}(P,Q)} \int_\Omega \norm{x - T(x)}^2 \mathrm{d}P(x),
\end{equation}
where $\mathcal{T}(P,Q)$ is the set of measurable maps $T : \Omega \to \Omega$ \emph{pushing forward} $P$ to $Q$, that is $Q(M) = P(T^{-1}(M))$ for every measurable set $M \subseteq \Omega$. This property, denoted by $T_\sharp P = Q$, means that if a random variable $X$ follows the distribution $P$ then its image $T(X)$ follows the distribution $Q$. According to Theorem 2.12 in \citep{villani2003topics}, originally demonstrated in \citep{cuesta1989notes,brenier1991polar}, when $P$ and $Q$ admit densities with respect to the Lebesgue measure and have finite second-order moments, then there exists a unique (up to $P$-negligible sets) solution to Problem~\eqref{monge}, which we denote by $T_0$.

Due to their transparent mathematical formulation and well-established theory, optimal transport maps became popular in many applications from statistics-related fields, where one aims at modeling shifts between distributions. This includes multivariate-quantile analysis \citep{beirlant2020center,hallin2021distribution}, signal analysis \citep{kolouri2017optimal}, domain adaptation \citep{courty2014domain,seguy2018large,redko2019optimal}, transfer learning \cite{gayraud2017optimal}, fairness in machine learning \citep{gordaliza2019obtaining,black2020fliptest}, and counterfactual reasoning \citep{de2021transport, berk2021improving}. However, in such practical frameworks, one typically does not have access to the true distributions $P$ and $Q$ but to independent samples $x_1,\ldots,x_n \sim P$ and $y_1,\ldots,y_n \sim Q$. This raises the question of constructing a tractable approximation of the solution $T_0$ on the basis of these empirical observations. The simplest way to compute an empirical optimal transport map from data points is to solve Problem~\eqref{monge} between the empirical measures $P_n := n^{-1} \sum^n_{i=1} \delta_{x_i}$ and $Q_n := n^{-1} \sum^n_{i=1} \delta_{y_i}$ instead of $P$ and $Q$. Implementing this solution suffers from three main drawbacks. The first one is the \emph{computational cost}, since it requires at least $O(n^3)$ operations to compute the empirical optimal transport map \citep{peyre2019computational}. The second is the \emph{memory cost}, since this map is typically stored as an $n \times n$ matrix. As a consequence of these two issues, this approach does not scale well with the size of the dataset. The third limitation of the empirical map is its \emph{inability to generalize} to new out-of-sample observations: by construction it is only matching the set $\{x_1, \ldots, x_n\}$ to $\{y_1,\ldots,y_n\}$.

These practical drawbacks triggered a vast literature on continuous approximations of optimal transport maps. The proposed mappings all come with different practical limitations, theoretical guarantees, and experimental performances. On the one hand, a wide range of these constructions provably converge in some sense to the true map $T_0$ as $n$ increases to infinity, making them consistent estimators. The so-called plug-in estimators, such as the ones proposed in \citep{beirlant2020center,hallin2021distribution,manole2021plugin}, extend the empirical solution to the whole domain $\Omega$ by leveraging regularity assumptions. However, they still bear the burdens of computing and storing the empirical transport map. The smooth estimator introduced by \citet{hutter2021minimax} reaches near-optimal minimax convergence rate, but fails to be computationally tractable. In contrast, \citet{seguy2018large} and \citet{pooladian2021entropic} employed entropic regularization, a numerical scheme based on Sinkhorn's algorithm \citep{cuturi2013sinkhorn}, to build an implementable and scalable estimator. On the other hand, several papers proposed learning the optimal transport map through neural networks, leading to expressive approximations with high generalization power. Specifically, \citet{leygonie2019adversarial} and \citet{black2020fliptest} developed approximations based on a generative-adversarial-network (GAN) objective \citep{goodfellow2014generative,arjovsky2017wasserstein}. More recently, the use of input convex neural networks, building on the convexity of the optimal transport potential, has received a growing attention \citep{makkuva2020optimal,korotin2021wasserstein,huang2021convex}. However, while these neural-based mappings display strong experimental performances, they generally lack theoretical guarantees, in particular the statistical convergence. 


To sum-up, the literature has mostly addressed either theoretically grounded statistical estimators of optimal transport maps, but unsuitable for large-scale implementations, or efficient heuristic approximations, at the cost of statistical guarantees. In this paper, we propose a novel GAN-based estimator $G_n$ of $T_0$ which, under some assumptions, converges uniformly:
\begin{equation*}
    \norm{G_n - T_0}_{\infty} \xrightarrow[n \to +\infty]{a.s.} 0.
\end{equation*}
Our construction takes root in the approximation from \citep{black2020fliptest}, defined as the generator of a penalized Wassertein-GAN (WGAN) training problem \citep{arjovsky2017wasserstein}, and improve it by assuming a setting where the optimal transport map is Lipschitz and by leveraging recent theoretical and practical advances on Lipschitz neural networks \citep{anil2019sorting,tanielian2021approximating,bethune2021faces}. Formally, $G_n$ solves the following adversarial training:
\begin{equation}\label{gan_problem}
    \inf_{G \in \G_n} \Big\{ \norm{I-G}^2_{L^2(P_n)} + \lambda_n \sup_{D \in \D_n} \int D\left(\mathrm{d}(G_\sharp P_n) - \mathrm{d}Q_n\right) \Big\},\nonumber
\end{equation}
where $\D_n$ is a class of 1-Lipschitz discriminators providing a proxy for the Wasserstein-1 distance, and $\G_n$ is a class of Lipschitz generators parametrizing the space of feasible mappings. The positive parameter $\lambda_n$ governs the trade-off between minimizing the quadratic transportation cost, promoting the objective of the Monge problem \eqref{monge}, and minimizing the distance between the generated and the target distributions, enforcing the push-forward constraint.

The most similar papers to ours are the ones of \citet{seguy2018large} and \citet{pooladian2021entropic}, as they propose feasible estimators with statistical guarantees. We note two main differences. First, we do not rely on entropic regularization while still ensuring scalability to large datasets. Second, our estimator innovates by being defined as a neural network. In particular, \citet{seguy2018large} relies on a neural network in practice, but the statistical convergence holds for a theoretical estimator. Regarding theoretical guarantees, we lack the convergences rates provided in \citep{pooladian2021entropic}, but we prove a stronger result than \citet{seguy2018large} by ensuring the uniform convergence of the estimator. 

\paragraph{Outline.} The rest of the paper is organized as follows:
\begin{enumerate}
    \item Section \ref{lnn} introduces the necessary background on so-called \emph{GroupSort} neural networks, which became the gold standard to parametrize Lipschiz feed-forward neural networks. By studying the multivariate setting, we provide generalizations of the main approximation theorem from \citep{tanielian2021approximating}.
    \item Section \ref{gan} presents the technical assumptions of our framework, in particular the regularity of the optimal transport map, details construction of our GAN estimator, and states the statistical consistency theorem.
    \item Section \ref{num} focuses on the practical implementation of the estimator, and study its performance through a number of numerical experiments.
\end{enumerate}

\begin{figure}
    \centering
    \includegraphics[scale=0.8]{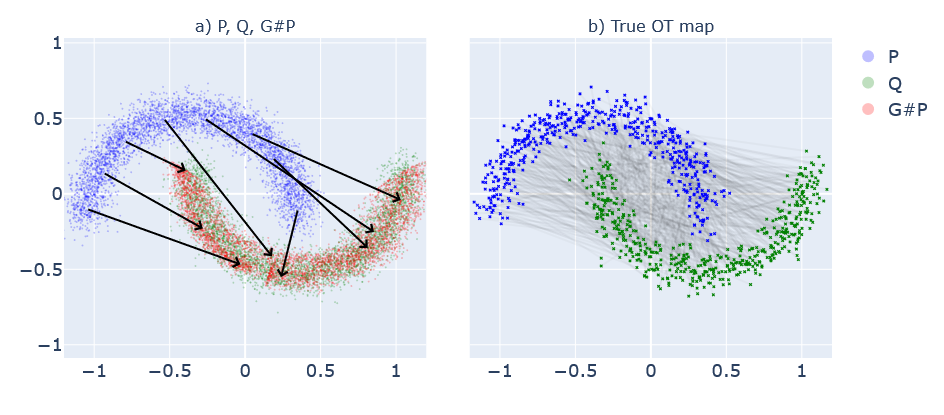}
    \caption{Estimation of the optimal transport map on the TwoMoons dataset. (a) GAN estimator $G$ after $800$ gradient steps on the generator, on the basis of $4,000$ points from each distribution. The black arrows represent the transport of specific points. (b) Empirical optimal transport map (discrete matching) between samples of size $500$.}
    \label{fig:failure_mode}
\end{figure}


\paragraph{Notations.} The absolute value of real numbers and the Euclidean norm of vectors are respectively given by $\abs{\cdot}$ and $\norm{\cdot}$. The notation $B_r$ refers to the centered Euclidean ball of $\R^d$ with radius $r>0$. We denote by $\operatorname{diam}(\Omega)$ the diameter of a set $\Omega \subseteq \R^d$. If $\Omega$ is a closed convex set, then $\mathcal{P}_\Omega$ stands for the projection onto $\Omega$. The support of a probability measure is given by $\operatorname{supp}(\cdot)$. In the following $\Omega_1 \subseteq \R^{d_1}$ and $\Omega_2 \subseteq \R^{d_2}$ denote two arbitrary subsets. For a function $F : \Omega_1 \to \Omega_2$ and $\mu$ a probability measure on $\Omega_1$, we write $\norm{F}_{L^2(\mu)} := \sqrt{\int_{\Omega_1} \norm{F(x)}^2 \mathrm{d}\mu(x)}$. The supremum norm of function is given by $\norm{\cdot}_{\infty}. $For some $L>0$, we write $\operatorname{Lip}_L(\Omega_1,\Omega_2)$ the set of $L$-Lipschitz functions from $\Omega_1$ to $\Omega_2$. For some $\alpha>0$, we call $\C^\alpha(\Omega_1,\Omega_2)$ the set of $\alpha$-Hölder functions from $\Omega_1$ to $\Omega_2$ and write $\norm{\cdot}_{\alpha,\infty}$ for the $\alpha$-Hölder norm of functions. For a differentiable function $F : \Omega_1 \to \Omega_2$, we call $F'$ its derivative, where for any $x \in \Omega_1$ the quantity $F'(x)$ is a $d_1 \times d_2$ matrix. For a real symmetric matrix $S$ and a real number $\gamma$, the relation $\gamma \preceq S$ indicates that all the eigenvalues of $S$ are greater than $\gamma$. The relation $\succeq$ is defined similarly.

\section{Lipschitz neural networks}\label{lnn}

The GAN estimator defined by \eqref{gan_problem} and further described in Section \ref{gan} requires generators and discriminators that are both Lipschitz. The question of imposing sharp Lipschitz constraints on neural networks has attracted much attention from the field of machine learning, especially with the popularization of WGANs which rely on 1-Lipschitz discriminators. In particular, gradient penalization \citep{gulrajani2017improved} has proven to be more efficient than the parameter-clipping approach originally proposed by \citet{arjovsky2017wasserstein}. In this paper, we focus on the recently introduced \emph{GroupSort} activation function to impose the Lipschitz constraint, which have proven to yield tighter estimates of 1-Lispchitz functions \citep{anil2019sorting,tanielian2021approximating}. We recall the necessary background on GroupSort-based networks, and show that their ability to approximate any bounded classes of Lipschitz functions holds for arbitrary output dimension. 

\subsection{Multivariate GroupSort neural networks}

We introduce GroupSort neural networks in a similar fashion to \citep{tanielian2021approximating}. In contrast, we consider a more general setting where the output dimension $p \geq 1$ is arbitrary. This difference is motivated by the optimal transport map being a multivariate function.

We write $\sigma_k$ for the GroupSort activation function of grouping size $k \geq 2$. By definition, it splits the pre-activation input into groups of size $k$, and then sorts each group by decreasing order. This operation is 1-Lipschitz, gradient-norm preserving and homogeneous \citep{anil2019sorting}. In this paper, we only address the grouping size 2. We call a GroupSort feed-forward neural network (with grouping size 2) any function $N_\theta : \R^d \to \R^p$ of the form 
\begin{equation}\label{GS}
    N_\theta = h_l \circ h_{l-1} \circ \ldots \circ h_1,
\end{equation}
where
\begin{align*}
    &h_1(x) := W_1 x + b_1 \mbox{ with } W_1 \in \R^{w_1 \times d}, b_1 \in \R^{w_1};\\
    &h_2(x) := W_2 \sigma_2(x) + b_2 \mbox{ with } W_2 \in \R^{w_2 \times w_1}, b_2 \in \R^{w_2};\\
    &\ldots\\
    &h_l(x) := W_l \sigma_2(x) + b_l \mbox{ with } W_l \in \R^{p \times w_{l-1}}, b_l \in \R^p.\\
\end{align*}
The integer $l \geq 1$ denotes the \emph{depth} of the network while the integers $\{w_1,\ldots,w_{l-1}\}$ refer to the \emph{widths} of the hidden layers $\{h_1,\ldots,h_{l-1}\}$. The widths are assumed to be divisible by $2$ (the grouping size). Additionally, we define $s := \sum^{l-1}_{i = 1} w_i$ the \emph{size} of the network. The parameter $\theta := (W_1,\ldots,W_l,b_1,\ldots,b_l) \in \Theta$ represents the \emph{weights} matrices and \emph{offset} vectors of $N_\theta$.

For a matrix $W$, let $\norm{W}_{\infty} := \sup_{\norm{x}_{\infty}=1} \norm{W x}_{\infty}$ and $\norm{W}_{2,\infty} := \sup_{\norm{x}=1} \norm{W x}_{\infty}$, where $\norm{x}_{\infty}$ denotes the maximum norm of vectors. Consider the following assumption on the parameters:
\begin{description}
    \item[(C)\namedlabel{compact}{\textbf{(C)}}] There exists a constant $C>0$ such that for all $(W_1,\ldots,W_l,b_1,\ldots,b_l) \in \Theta$,
\begin{align*}
    \norm{W_1}_{2,\infty} &\leq 1,\\
    \max(\norm{W_2}_{\infty},\ldots,\norm{W_l}_{\infty}) &\leq 1,\\
    \max(\norm{b_1}_{\infty},\ldots,\norm{b_l}_{\infty}) &\leq C.
\end{align*}
\end{description}
In the following, we denote by $\mathcal{N}^p_C(l,s)$ the class of GroupSort feed-forward neural networks with depth $l$, size $s$, output dimension $p$, satisfying Assumption \ref{compact} for the constant $C>0$. When the depth and size are arbitrary, we simply write $\mathcal{N}^p_C$. The following result is a trivial extension to the multivariate case of Lemma~1 in \citep{tanielian2021approximating}, stating that GroupSort neural networks satisfying Assumption \ref{compact} are 1-Lipschitz.

\begin{lemma}\label{lips}
For any $C>0$, $\mathcal{N}^p_C \subset \operatorname{Lip}_1(\R^d,\R^p)$.
\end{lemma}
Next, we study their ability to approximate Lipschitz continuous functions.

\subsection{Approximating Lipschitz continuous functions}

We now restrict the input domain to a \emph{compact} subset of $\R^d$ denoted by $\Omega$. The following lemma states that for a well-chosen $C$ the class $\mathcal{N}^1_C$ approximates with given precision any bounded subclass of $\operatorname{Lip}_1(\Omega,\R)$. It generalizes Theorem 2 in \citep{tanielian2021approximating} by providing the universal constant for which Assumption \ref{compact} is satisfied, and extending the result to any compact domain $\Omega$ while it was restricted to $[0,1]^d$.

\begin{theorem}\label{approx}
Let $\F \subseteq \operatorname{Lip}_1(\Omega,\R)$ be a class of functions such that $\sup_{f\in\mathcal{F}} \norm{f}_{\infty}\leq K_\F$ for some $K_\F > 0$. Set $\epsilon>0$ and $C := K_\F + \sqrt{d}(\sup_{x \in \Omega} \norm{x}+1) + \epsilon$. Then, for any $f \in \F$, there exists a neural network $N \in \mathcal{N}^1_C(l,s)$ where
\begin{equation*}
    l=O\left(d^2\log_2\left(\frac{2\sqrt{d}}{\epsilon}\right)\right)\ \text{and} \
    s=O\left(\left(\frac{2\sqrt{d}}{\epsilon}\right)^{d^2}\right),
\end{equation*}
such that $\norm{N - f}_{\infty} \leq \epsilon$.
\end{theorem}

The proof essentially follows that of \citet{tanielian2021approximating}. It generalizes some parts by tracking the bound on the offset vectors of the approximating network. Interestingly, Theorem~\ref{approx} can be extended to the case where the output is of dimension $p$.

\begin{theorem}\label{approx_d}
Let $\G \subseteq \operatorname{Lip}_1(\Omega,\R^p)$ be a class of functions such that $\sup_{g \in\mathcal{G}} \norm{g}_{\infty}\leq K_\G$ for some $K_\G>0$. Set $\varepsilon>0$ and $C = K_\G + \sqrt{d} (\sup_{x \in \Omega} \norm{x} +1) + \varepsilon$. Then, for any $g \in \G$ there exists a neural network $N \in \mathcal{N}^p_C(l,s)$ where
\[
    l=O\left(d^2\log_2\left(\frac{2\sqrt{d}\sqrt{p}}{\epsilon}\right)\right),\ \ \text{and}\ \ 
    s=O\left(p\left(\frac{2\sqrt{d}\sqrt{p}}{\epsilon}\right)^{d^2}\right),
\]
such that $\norm{N - g}_{\infty} \leq \varepsilon$.
\end{theorem}

The proof amounts to applying Theorem~\ref{approx} to the univariate function along each dimension. Note that Theorems \ref{approx} and \ref{approx_d} can be extended to approximate $L$-Lipschitz functions, for an arbitrary $L>0$, by multiplicating by $L$ the output later of 1-Lipschitz neural networks. This remark will be useful to approximate the optimal transport map, assumed to be $L$-Lipschitz.

\section{GAN estimator}\label{gan}

In this section, we address the construction of an estimator of the optimal transport map, and show its uniform convergence as the sample size increases to infinity.

\subsection{Optimal transport setup}

Set $P$ and $Q$ two measures on $\R^d$ admitting densities with respect to the Lebesgue measure and with finite second-order moments. We aim at estimating with a GroupSort neural network the unique optimal transport map $T_0$ between $P$ and $Q$ through the knowledge of the empirical distributions $P_n$ and $Q_n$. As mentioned in the introduction, we consider a setting where the optimal transport map $T_0$ is Lipschitz.

As in the previous section, $\Omega \subset \R$ is a compact set, and we denote by $\Omega_P := \operatorname{supp}(P)$ the \emph{source} domain and $\Omega_Q := \operatorname{supp}(Q)$ the \emph{target} domain. Then, we let $L \geq 2$ and make the following assumptions:

\begin{description}
    \item[(S1)\namedlabel{S-source}{\textbf{(S1)}}] The source domain $\Omega_P \subseteq B_L$ is a bounded and connected Lipschitz domain. The measure $P$ admits a density $\rho$ with respect to the Lebesgue measure such that $L^{-1} \leq \rho(x) \leq L$ for almost every $x \in \Omega_P$.
    \item[(S2)\namedlabel{S-map}{\textbf{(S2)}}] Let $\Tilde{\Omega}_P$ denote a convex set with Lipschitz boundary such that $\Omega_P + B_{L^{-1}} \subseteq \Tilde{\Omega}_P \subseteq B_L$. The optimal transport map $T_0$ is a differentiable function from $\Tilde{\Omega}_P$ to $\R^d$ such that $T_0 = \nabla f_0$ where $f_0 : \Tilde{\Omega}_P \to \R^d$ is a differentiable convex function. Additionally it satisfies:
    \begin{itemize}
        \item[(i)] $T_0 \in C^{2}(\Tilde{\Omega}_P,\R^d)$ such that $\norm{T_0}_{2,\infty} \leq L$;
        \item[(ii)]$L^{-1} \preceq T'_0(x) \preceq L$ for all $x \in \Tilde{\Omega}_P$.
    \end{itemize}
\end{description}

These are the same hypothesis as in Section 5 from \citep{hutter2021minimax}, specified with a Hölder regularity $\alpha$ equals to 2. This makes our setting milder, as we do not require the optimal transport map to be highly regular. Assumptions \ref{S-source} and \ref{S-map} ensure the existence of a near-optimal minimax estimator of $T_0$, which play a key role in the proof of our estimator's consistency. Note that, without loss of generality, we can consider that $P$ and $Q$ are measures on a compact set $\Omega \subset \R^d$ sufficiently large to contain $B_L$. Then, Assumption \ref{S-map} implies that $T_0 \in \operatorname{Lip}_L(\Omega,B_L)$.

Now that the optimal transport problem is properly specified, we turn to the GAN architecture through which our estimator is defined.


\subsection{GAN setup}

The optimal transport map $T_0$ satisfies two objectives: it is constrained to pushing-forward $P$ to $Q$, that is ${T_0}_\sharp P = Q$; it minimizes the quadratic transportation cost $\norm{I-T_0}^2_{L^2(P)}$. Due to the push-forward condition, $T_0$ can be regarded as a generative model. This observation is the foundation of the approximation of \citet{black2020fliptest}. They proposed to regularize the WGAN objective function, promoting only the push-forward condition, with an optimal transport penalty on the generator. We proceed similarly, with three critical differences. First, we penalize the quadratic transportation cost with the push-forward condition instead of the converse. Second, we employ GroupSort neural networks to implement the discriminator and generator. Third, because we aim at proving the statistical convergence of the generator, we emphasize for all the objects involved in the GAN their dependence to the sample size $n$, including the penalty weight. 



\subsubsection{Discriminator}

In the WGAN framework, the \emph{discriminator} $D : \R^d \to \R$ is a neural network defining a proxy for the Wasserstein-1 distance, while the \emph{generator} $G : \R^d \to \R^d$ is a neural network minimizing this proxy between ${G}_\sharp P_n$ and $Q_n$, thereby aiming at generating $Q$ from $P$.

We recall that the Wasserstein-1 distance between two measures $\mu$ and $\nu$ on $\Omega$ is defined as,
\begin{equation*}
    \W(\mu,\nu) := \inf_{\pi \in \Pi(\mu,\nu)} \int_{\Omega \times \Omega} \norm{x-y} \mathrm{d}\pi(x,y),
\end{equation*}
where $\Pi(\mu,\nu)$ is the set of couplings with $\mu$ as first marginal and $\nu$ as second marginal. Interestingly, this distance enjoys the following dual formulation, known as the Kantorovich-Rubinstein formula \citep{kantorovich1958space}. According to the Particular Case~5.15 of Theorem~5.9 in \citep{villani2008optimal}, this can be written as:
\begin{equation}\label{dual}
    \W(\mu,\nu) = \sup_{f \in \operatorname{Lip}_1(\Omega,\R)} \int f (\mathrm{d}\mu - \mathrm{d}\nu). 
\end{equation}
The key idea of WGAN is to approximate this distance by computing the supremum over a class of neural networks included in $\operatorname{Lip}_1(\Omega,\R)$. The larger the class, the better the approximation. Originally, this was done by clipping, thresholding the weights of the network, leading to a coarse approximation of the Wasserstein distance. Later, several papers showed that using GroupSort neural networks led to sharper approximations \citep{anil2019sorting,biau2021some}.

Actually, note that if $f$ is an optimal function in Problem~\eqref{dual}, then the function $f+c$ for any constant $c$ is also an optimal solution. As a consequence, we can without loss of generality restrict the set of feasible potentials to 1-Lipschitz functions taking the value zero at a given arbitrary anchor point $x_0 \in \Omega$. Formally, let's define
\begin{equation}\label{F}
    \F := \{f \in \operatorname{Lip}_1(\Omega,\R) \mid f(x_0)=0\}.
\end{equation}
Then we can write,
$$
    \W(\mu,\nu) = \sup_{f \in \F} \int f (\mathrm{d}\mu - \mathrm{d}\nu). 
$$
The interest of this formulation is that the feasible potentials now belongs to a bounded subclass of Lipschitz functions.
\begin{lemma}\label{lm:F}
Let $\F$ be defined as in Equation~\eqref{F}. Then,
\[
K_\F := \sup_{f \in \F} \norm{f}_{\infty} \leq \operatorname{diam}(\Omega).
\]
\end{lemma}
Thus, Theorem~\ref{approx} entails that they can be approximated by GroupSort neural networks with specific depth and size. Following this remark, we define for each sample size $n$ the class of feasible discriminators $\D_n$ as well-chosen GroupSort neural networks. Specifically, the discriminators are defined as in the next assumption.

\begin{description}
    \item[(G1)\namedlabel{discriminator}{\textbf{(G1)}}] Set a sequence of positive numbers $\{\epsilon_n\}_{n \in \N}$ such that $\lim_{n \to +\infty} \epsilon_n = 0$, and a sequence of constants $\{C_n\}_{n \in \N}$ defined as
    $$
    C_n := \operatorname{diam}(\Omega) + \sqrt{d}(\sup_{x \in \Omega} \norm{x}+1) + \epsilon_n.
    $$
    For every $n \in \N$, define $\D_n := \mathcal{N}^1_{C_n}(l_n,s_n)$ where,
    \[
    l_n=O\left(d^2\log_2\left(\frac{2 \sqrt{d}}{\epsilon_n}\right)\right),\ \ 
    \text{and}\ \
    s_n=O\left( \left(\frac{2 \sqrt{d}}{\epsilon_n}\right)^{d^2}\right).
    \]
\end{description}

Then, we approximate the Wasserstein-1 distance through the following integral probability metric:

\begin{equation}\label{IPM}
     \W_n(\mu,\nu) := \sup_{D \in \D_n} \int D (\mathrm{d}\mu - \mathrm{d}\nu).
\end{equation}

An important consequence of Assumption \ref{discriminator} through Lemma~\ref{lm:F} and Theorem~\ref{approx} is that $\bigcup_{n \in \N} \D_n$ is dense in $\F$, rendering $\W_n$ asymptotically close to $\W$ as $n$ increases to infinity. Note that the sequence $\{\epsilon_n\}_{n \in \N}$ characterizes the rate at which the class $\D_n$ approximates $\F$. Now that we have properly defined the discriminators, we focus on the generators.  

\subsubsection{Generator}

\begin{figure*}[tb]
    \centering
    \includegraphics[scale=0.4]{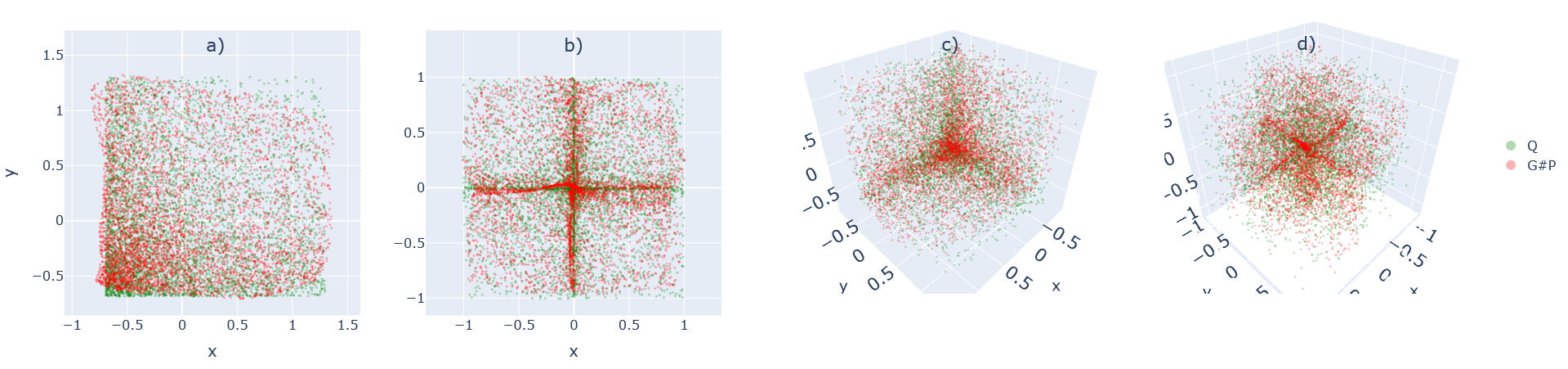}
    \caption{Visualisation of $G_\sharp P$ and $Q := {T_0}_\sharp P$ with $10,000$ points. $P$ is the uniform distribution on $[-1,1]^d$. The generator is trained for $120$ gradient steps. The Figures (a)-(b) corresponds to $d=2$. The Figures (c)-(d) corresponds to $d=3$. In Figures (a)-(c), we defined $T_0$ by coordinate-wise application of $x \mapsto \frac{1}{1.18}(\exp{x}-1.18)$. In Figures (b)-(d), we defined $T_0$ by coordinate-wise application of $x \mapsto x^2\text{sign}(x)$.}
    \label{fig:exp2d}
\end{figure*}

On the contrary to a standard WGAN, the generator must additionally minimize the quadratic transportation cost in order to approach the optimal transport map $T_0$. Let us denote by $\G_n$ the class of feasible generators, which will be specified later. A naive formulation for our estimator $G_n \in \G_n$ would be,
\begin{equation*}
    G_n \in \argmin_{G \in \G_n \text{ s.t. } {G}_\sharp P_n = Q_n} \norm{I-G}^2_{L^2(P_n)}.
\end{equation*}
However, since the push-forward condition is intractable as such, we replace it by a penalty term based on the neural proxy of the Wasserstein-1 distance. Formally, we set $\lambda_n > 0$ a regularization weight and we define the GAN estimator $G_n$ as an optimal solution to Problem~\eqref{gan_problem}, that is
\begin{equation*}
    G_n \in \argmin_{G \in \G_n} \L_n(G),
\end{equation*}
where
\begin{equation*}
    \L_n(G) := \norm{I-G}^2_{L^2(P_n)} + \lambda_n \W_n(G_\sharp P_n,Q_n).
\end{equation*}
We note that Problem \eqref{gan_problem} is well-posed under mild conditions.
\begin{proposition}\label{existence}
If $\D_n \subseteq \operatorname{Lip}_1(\Omega,\R)$ and $\G_n$ is compact, then Problem \eqref{gan_problem} admits solutions.
\end{proposition}
This result is a direct consequence of the Lipschitz continuity of the loss function $\L_n$, which we demonstrate in the proof.

At this stage, we should make further assumptions on $\G_n$ to exploit the smoothness of the optimal transport problem. Let us define
\begin{equation}\label{G}
    \G := \operatorname{Lip}_L(\Omega,B_L),
\end{equation}
which is a class of bounded Lipschitz functions.
\begin{lemma}
    Let $\G$ be defined as in Equation~\eqref{G}. Then,
    \[
    K_\G := \sup_{g \in \G} \norm{g}_{\infty} \leq L \operatorname{diam}(\Omega) + \sup_{x \in \Omega} \norm{x}.
    \]
\end{lemma}
Critically, under Assumption \ref{S-map}, the solution $T_0$ belongs to $\G$, and as such can be approximated by GroupSort neural networks according to Theorem \ref{approx_d}. This motivates the following conditions on the set of feasible generators $\G_n$:
\begin{description}
    \item[(G2)\namedlabel{generator}{\textbf{(G2)}}] Set $\{\varepsilon_n\}_{n \in \N}$ a sequence of positive numbers such that $\lim_{n \to +\infty} \varepsilon_n = 0$, and a sequence of constants $\{C_n\}_{n \in \N}$ defined as
$$
C_n := L \operatorname{diam}(\Omega) + (\sqrt{d}+1)\sup_{x \in \Omega} \norm{x} + \sqrt{d} + \varepsilon_n.
$$
For every $n \in \N$, we define $\G_n$ as
$$
    \{ x \in \Omega \mapsto \mathcal{P}_{B_L}(L \times N(x)), N \in \mathcal{N}^d_{C_n}(l_n,s_n)\}
$$
where,
\[
    l_n=O\left(d\log_2\left(\frac{2d}{\varepsilon_n}\right)\right),\ \ 
\text{and}\ \ 
    s_n=O\left(d \left(\frac{2 d}{\varepsilon_n}\right)^{d^2}\right).
\]
\end{description}
Defined as such, $\G_n$ is included in $\G$. The idea behind Assumption \ref{generator} is similar to that of Assumption \ref{discriminator}. In particular, the condition on the depth and size of the networks guarantees through Theorem~\ref{approx_d} that $\G_n$ asymptotically fills $\G$ at speed $\varepsilon_n$, allowing to recover $T_0$ at the limit.

\subsection{Main theorem}

The convergence of $\{G_n\}_{n \in \N}$ towards $T_0$ revolves around two antagonistic conditions. Instinctively, the sequence of regularization weights $\{\lambda_n\}_{n \in \N}$ must tend to infinity in order to impose the push-forward condition at the limit. Concurrently, the sequence of feasible generators $\{\G_n\}_{n \in \N}$ must fill $\G$ sufficiently fast. This corresponds to the following assumptions:
\begin{description}
    \item[(G3)\namedlabel{lambda}{\textbf{(G3)}}] The sequence $\{\lambda_n\}$ is such that $\lim_{n \to +\infty} \lambda_n = +\infty$ and
    $$
        \lambda_n = \begin{cases} o\left(n^{\frac{1}{d}}\right) & \mbox{if } d > 2,\\
        o\left(n^\frac{1}{2}/ \log n\right) & \mbox{if } d=2,\\
        o\left(n^\frac{1}{2}/ \sqrt{\log n}\right) & \mbox{if } d=1.\\
        \end{cases}
    $$
    \item[(G4)\namedlabel{epsilon}{\textbf{(G4)}}] The sequence $\{\varepsilon_n\}_{n \in \N}$ from Assumption \ref{generator} is such that, $\varepsilon_n = o\left(\frac{1}{\lambda_n}\right)$.
\end{description}
We are now ready to state our main theorem.
\begin{theorem}\label{main}
Let $P$ and $Q$ be such that the smoothness assumptions \ref{S-source} and \ref{S-map} on the optimal transport problem hold, and denote by $T_0$ the (almost everywhere) unique optimal transport map between $P$ and $Q$. Suppose that the GAN problem satisfies Assumptions \ref{discriminator}, \ref{generator}, \ref{lambda} and \ref{epsilon}. Then, for $G_n$ defined as a solution to Problem~\eqref{gan_problem} we have
$$
    \norm{G_n-T_0}_{\infty} \xrightarrow[n \to +\infty]{a.s.} 0.
$$
\end{theorem}
To the best of our knowledge, this is the first statistical consistency result for a neural-network-based optimal transport map. We leave the analysis of consistency rates for future work. In particular, we could obtain sharper results by imposing conditions on the parameter $\epsilon_n$ which characterizes the rate at which the discriminators $\D_n$ approximate the 1-Lipschitz potentials, and by leveraging stronger regularity assumptions on $T_0$. The proof is quite technical; the convergence of $\lambda_n$ to infinity prevents from using classical empirical process techniques. Instead, we rely on more analytical arguments based on the relative compactness properties of Lipschitz functions. Moreover, we note that the proof still holds for more general classes of generators as long as they maintain certain universality properties and have a Lipschitz constant that can be controlled. This is one of the main strengths of GroupSort neural networks: they can sharply approximate any classes of bounded Lipschitz functions with the same Lipschitz constant.

\section{Numerical experiments}\label{num}

The rest of the paper addresses the implementation of our method, and showcases experimental results. Specifically, we do not try to illustrate the convergence rate of the estimator, which is yet to be found, but instead focus on the efficiency and practicality of our GAN-based optimal transport map.

\subsection{Implementation}

In the following experiments, we use ($\cdot\rightarrow 80\rightarrow 80\rightarrow 80\rightarrow\cdot$) densely connected neural networks with GroupSort activation functions for both the generator and the discriminator. We implement GroupSort using Deel-Lip library\footnote{\url{https://deel-lip.readthedocs.io}}. The 1-Lipschitz constraint is enforced through projections onto a parameter space satisfying Assumption~\ref{compact}. The output layer of the generator is multiplied by $L$ to be made $L$-Lipschitz. Critically, since this constant is unknown in practice, we must rely on a large-enough user-defined upper bound. We use Adam with default parameters for the optimization. All experiments have been run on personal workstation with 32GB RAM and NVIDIA Quadro RTX 8000 48GB GPU.

\begin{algorithm}[ht]
\caption{GAN learning of the optimal transport map}\label{alg:w2param}
\begin{algorithmic}
\STATE {\bfseries Input:} source distribution $P$, target distribution $Q$, regularization parameter $\lambda$, discriminator $\{D_{\psi}\}_{\psi \in \Psi}$, generator $\{G_{\phi}\}_{\phi \in \Phi}$, respective learning rates $\eta_D$ and $\eta_G$, minibatch size $m$
\REPEAT
    \REPEAT
        \STATE Sample minibatchs: $\{x_i\}^m_{i=1} \sim P$, $\{y_i\}^m_{i=1}\sim Q$
        \STATE Define cost function:
        $$\mathcal{W}_D(\psi)\mathrel{\mathop:}=\frac{1}{m}\sum^m_{i=1} D_{\psi}(G_{\phi}(x_i))-\frac{1}{m}\sum^m_{i=1} D_{\psi}(y_i)$$
        \STATE Projected gradient ascent step on discriminator: $$\psi\gets \mathcal{P}_{\Psi}\left(\psi+\eta_D \nabla_{\psi}\mathcal{W}_D(\psi)\right)$$
    \UNTIL{convergence of $D_{\psi}$}
    \STATE Sample minibatch: $\{x'_i\}^m_{i=1} \sim P$
    \STATE Define cost functions:
    \begin{align*}
    \begin{aligned}
    \mathcal{W}_G(\phi)\mathrel{\mathop:}=&\frac{1}{m}\sum^m_{i=1} D_{\psi}(G_{\phi}(x'_i))\\
    \mathcal{C}(\phi)\mathrel{\mathop:}=&\frac{1}{m}\sum^m_{i=1}\norm{x'_i-G_{\phi}(x'_i)}^2
    \end{aligned}
    \end{align*}
    \STATE Projected gradient descent step on generator: $$\phi\gets \mathcal{P}_{\Phi}\left(\phi-\eta_G \nabla_{\phi} ( \C(\phi) + \lambda \mathcal{W}_G(\phi) )\right)$$
\UNTIL{convergence of $G_{\phi}$}
\end{algorithmic}
\end{algorithm}

The learning procedure is detailed in Algorithm~\ref{alg:w2param}. In contrast to a WGAN, the generator loss includes the quadratic transportation cost. It also differs from the procedure proposed in \citep{black2020fliptest} by implementing a sharper weight projection than clipping.
  


\subsection{Experimental results}

\begin{figure}[tb]
    \centering
    \includegraphics[scale=0.3]{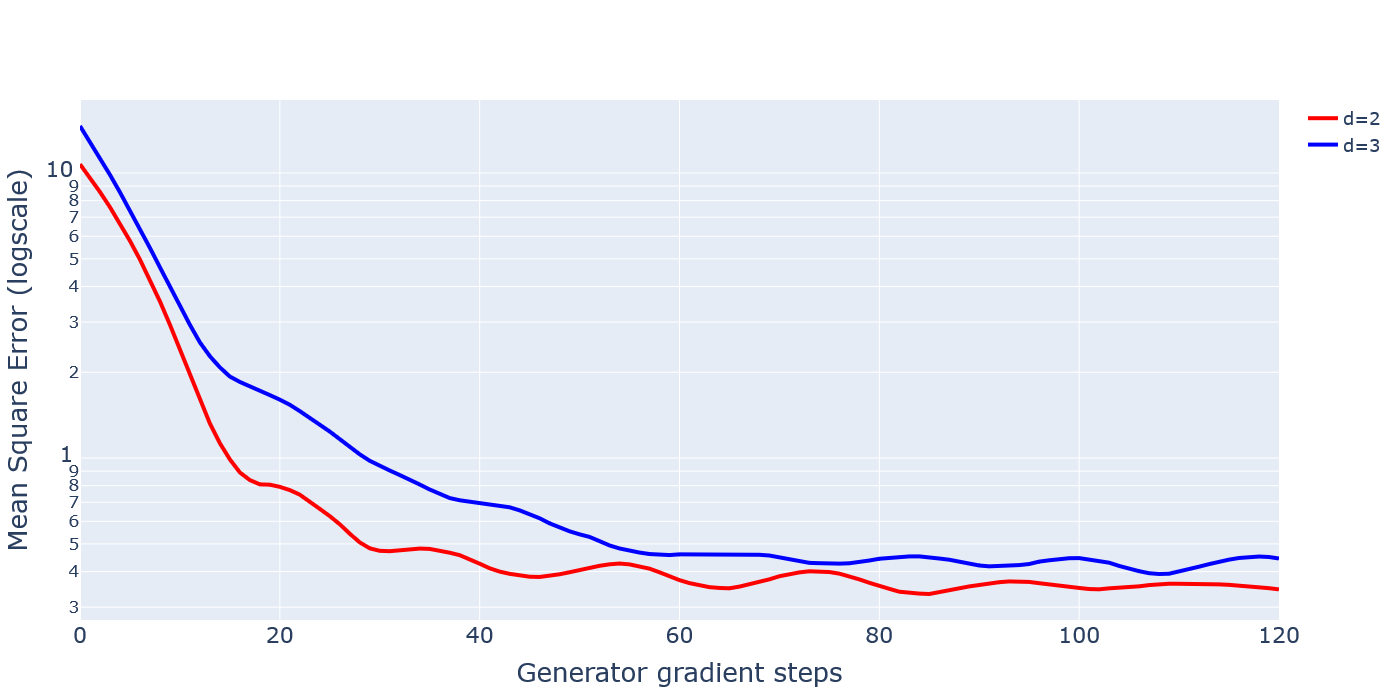}
    \caption{Evolution of the mean square error $\|T_0-G\|_2^2$ during the learning process as function of the number of gradient steps on generator with batch size $512$, for $x \mapsto \frac{1}{1.18}(\exp{x}-1.18)$. The number of samples used is proportional to the number of steps.}
    \label{fig:convergence_speed}
\end{figure}
  
We evaluate how close the trained generator $G$ is to the optimal transport map $T_0$. Recall that our construction, as in \citep{hutter2021minimax,pooladian2021entropic}, is tailored to settings where the optimal transport map is at least Lipschitz, hence continuous. This excludes in particular target distributions with disconnected supports. Firstly, we address a setting where the true optimal transport map $T_0$ in unknown. Figure~\ref{fig:failure_mode} benchmarks the GAN estimator against the empirical optimal transport map on the TwoMoons dataset. We used the POT library to compute the discrete matching \citep{flamary2021pot}. It shows that the generator faithfully matches the two moons with respect to the quadratic transportation cost.

Secondly, we consider synthetic examples for which $T_0$ has an explicit formula. We follow the protocol adopted in the aforementioned papers by defining $P$ as the uniform distribution on the hypercube $[-1, 1]^d$ and setting $Q := {T_0}_\sharp P$, where $T_0:\R^d\to\R^d$ is obtained by applying a monotone scalar function coordinate-wise. The combination of McCann's theorem \citep{mccann1995}, stating that there exists a unique gradient of a convex function achieving the push-forward between two Lebesgue-absolutely-continuous distributions, and Theorem~2.12 in \citep{villani2003topics}, stating that an optimal transport map coincide almost-everywhere with the gradient of a convex function, ensures that $T_0$ constructed as such is the (almost everywhere unique) optimal transport map between $P$ and $Q$. Note that for practical reasons, we choose $T_0$ such that $Q$ is a distribution with zero mean and width less than $2$: normalizing the input and output distributions of a neural network ensures faster convergence. The result are illustrated in Figure~\ref{fig:exp2d}.

  


Additionally, we investigate in Figure~\ref{fig:convergence_speed} the evolution of the mean square error between the generator $G$ and the optimal transport map $T_0$ as the learning process goes on. It confirms that the optimization scheme has the expected behaviour. Furthermore, since the mean square error is evaluated on an independent sample to the training set, it illustrates the generalization ability of the learnt map.  

  


\section{Conclusion}

The method we propose has the advantage of providing a theoretically sound and feasible estimation of the optimal transport map whose statistical convergence can be mathematically certified. Theorem~\ref{main} proves its consistency, while Section~\ref{num} highlights its feasibility and illustrates its ability to learn the underlying map. This renders this estimator suitable for many applications where guarantees of convergence are required while maintaining a high level of computational performance. 

Additionally, we extended in Section~\ref{lnn} the established theory on approximating Lipschitz continuous functions by GroupSort neural networks to the multivariate case. This also opens new lines of inquiry for further applications of these networks, such as imposing regularity properties on generative models. Finally, our statistical framework and mathematical proofs addressed several interesting problems at the frontier between neural network modeling and statistics. We hope this effort will contribute to bridge the gap between deep learning and statistical theory. 




\bibliography{references}
\bibliographystyle{abbrvnat}

\newpage
\onecolumn

\appendix



\section{Proof of Theorem \ref{approx}}

Let $f \in \F \subset \operatorname{Lip_1}(\Omega,\R)$ such that $\sup_{f \in \F} \norm{f}_\infty \leq K_\F$. The idea is to generalize Theorem~2 in \citep{tanielian2021approximating}, restricted to 1-Lipschitz functions on the hypercube $[0,1]^d$, to functions on the arbitrary compact set $\Omega$. To this end, we first transform $f$ into a 1-Lipschitz function on the hypercube $[0,1]^d$.

Since $\Omega$ is compact then there exists some $R>0$ such that $\Omega \subset [-R,R]^d$.  Kirszbraun’s theorem, see for instance Theorem 2.5 in \citep{heinonen2005lectures}, implies that we can extend $f$ on $[-R,R]^d$ while preserving the 1-Lipschitz property. Concretely, there exists a function $\tilde{f}\in \text{Lip}_1( [-R,R]^d,\R)$ such that $\tilde{f}(x)=f(x)$ for all $x \in \Omega$. 

Now, we transform the extension $\tilde{f}$ into a 1-Lipschitz function on the hypercube $[0,1]^d$. This requires to translate and scale the inputs. Set $x_R = R \cdot \mathbf{1}$ where $\mathbf{1} := (1,\ldots,1) \in \R^d$, and define $f_R(x) := \frac{1}{2R}\tilde{f}(2Rx-x_R)$ as a function on $[0,1]^d$. Theorem~2 in \citep{tanielian2021approximating} yields that, for every $\epsilon>0$, there exists a neural network $N$ of the form \eqref{GS} satisfying Assumption \ref{compact} defined on $[0,1]^d$ whose depth and size are respectively
$$
l= O\left(d^2\log_2\left(\frac{2\sqrt{d}}{\epsilon}\right)\right) \ \text{and} \ s={O\left(\left(\frac{2\sqrt{d}}{\epsilon}\right)^{d^2}\right)},$$ 
such that 
\begin{equation}\label{eq:fRN}
    \sup_{x \in[0,1]^d} \abs{f_R(x)-N(x)} <\epsilon.
\end{equation} 
However, \citet{tanielian2021approximating} never clearly specified a universal bound $C$ for which Assumption \ref{compact} was satisfied, which is necessary to conclude. To find such a bound, we detail how they constructed the GroupSort neural network $N$ approximating $f_R$. First, note that according to Theorem 5.1 in \citep{he2020relu}, any 1-Lipschitz piecewise-affine function $q$ defined on a compact set can be written as,
\begin{align}\label{eq:represent_f_2}
    q(x) = \max_{1\leq s\leq m}\min_{i \in I_s}(a_i \cdot x + c_i),
\end{align}
where for any $1\leq s\leq m$, $I_s$ is a subset of $\{1,\dots,m\}$ and $\norm{a_i} \leq 1$. Second, following the proof of Theorem 2 in \citep{tanielian2021approximating}, one can find a 1-Lipschitz piecewise-affine function $q$ such that $\norm{q - f_R} \leq \epsilon$. Finally, Theorem~1 in \citep{tanielian2021approximating}, states that $q$ can be represented by a neural network $N$ of the form \eqref{GS} with depth $l$ and size $s$. Critically, the representing $N$ is built with weights $(W_1,\ldots,W_l,b_1,\ldots,b_l)$ such that the offset vectors of $N$ are all equal to zero except $b_1$. More precisely, the coefficients of $b_1$ are the constants $c_1,\ldots,c_m$ from the representation \eqref{eq:represent_f_2}. This entails that $\max_{1 \leq i \leq l} \norm{b_i}_{\infty} \leq \max_{1 \leq i \leq m} \abs{c_i}$. Hence, bounding the constants in \eqref{eq:represent_f_2} will bound the offsets vectors in \eqref{GS}. To find a bound on the constants, we rely on the following lemma.
\begin{lemma}\label{Lemma:universal}
Let $ f_1 \in \operatorname{Lip}_1([0,1]^d,\R)$ and $f_2$ be a 1-Lipschitz piecewise-linear function such that $\norm{f_2-f_1}_{\infty} <\epsilon$. Then, $f_2$ can be expressed in the form \eqref{eq:represent_f_2} with $$\max_{1 \leq i \leq m} \abs{c_i} \leq \norm{f_1}_{\infty}+\epsilon+\sqrt{d}.$$
\end{lemma}
\begin{proof}
Note that we can suppose without loss of generality that for any $k \in \{1,\ldots,m\}$ there exists a point $x_k \in \Omega$ such that $f_2(x_k) = a_k \cdot x_k+c_k$, otherwise this index is meaningless and we can eliminate it. Since $\norm{a_k} \leq 1$, we have that $\abs{c_k}\leq \norm{f_2}_{\infty}+\sup_{x \in [0,1]^d} \norm{x}$. We conclude using the fact that $\norm{f_2}_{\infty} \leq \norm{f_1}_{\infty} + \epsilon$.
\end{proof}
This implies that the $\epsilon$-approximation $q$ of $f_R$ is such that $\max_{1 \leq i \leq m} \abs{c_i} \leq K_\F + \epsilon + \sqrt{d}$, and that consequently, the neural network $N$ approximating $f_R$ belongs to $\mathcal{N}^1_{C_0}(l,s)$ with $C_0 = K_\F + \epsilon + \sqrt{d}$.

Now, recall the the objective is to construct a neural network approximating $f$. Note that, after a change of variable, \eqref{eq:fRN} can be written as
$$
\sup_{x \in[-R,R]^d} \abs{\tilde{f}(x)-2R N\left(\frac{x+x_R}{2R}\right)  }<2R\epsilon.
$$ 
Since the activation functions are GroupSort, hence homogeneous, we have that $2R N\left(\frac{x+x_R}{2R}\right) = N(x+x_R)$. This leads to
$$
\norm{f - N_R}_{\infty} \leq \sup_{x \in[-R,R]^d} \abs{\tilde{f}(x) - N_R(x)} < 2R\epsilon,
$$
Finally, remark that the neural network $N_R : x \mapsto N(x+x_R)$ belongs to $\mathcal{N}^1_C(l,s)$ with $C = \sqrt{d}R + C_0$ that is $\sqrt{d}(R+1) + K_\F + \epsilon$. Setting $R = \sup_{x \in \Omega} \norm{x}$ completes the proof.

\section{Proof of Theorem \ref{approx_d}}

Let $g \in \G \subset \operatorname{Lip}_1(\Omega,\R^p)$ such that $\sup_{g \in \G} \norm{g}_{\infty} = K_\G > 0$. We generalize Theorem \ref{approx} to $\R^p$-valued output by approximating $g$ along each dimension by a GroupSort neural network. The function $g$ can be written as $(g_1, \ldots, g_p)$ where $g_i \in \operatorname{Lip}_1(\Omega,\R)$ and $\norm{g_i}_{\infty} \leq K_\G$ for every $1 \leq i \leq p$. Then, we know from Theorem \ref{approx} that there exists a neural network $N^i \in \mathcal{N}^1_{C}$ where $C = K_\G + \sqrt{d}(\sup_{x \in \Omega} \norm{x} +1) + \varepsilon$, whose depth and size are respectively
$$
l= O\left(d^2\log_2\left(\frac{2\sqrt{d}}{\epsilon}\right)\right) \ \text{and} \ s={O\left(\left(\frac{2\sqrt{d}}{\epsilon}\right)^{d^2}\right)},$$ 
such that,
$$
    \norm{g_i - N^i}_{\infty} \leq \varepsilon.
$$
We build the $\R^p$-valued neural network $N = (N^1, \ldots, N^p)$. Then, for any $x \in \Omega$,
$$
\norm{g(x) - N(x)}^2 = \sum^p_{i=1} \abs{g_i(x) - N^i(x)}^2 \leq p \varepsilon^2.
$$
As a consequence, $\norm{g - N}_{\infty} \leq \sqrt{p} \varepsilon$.
To conclude, note that $N$ and has depth $l$ and size $p \times s$. Moreover, it satisfies Assumption \ref{compact} for the constant $C$, as the weight matrices and offset vectors of $N$ are obtained by concatenation of the ones of the $N^i$, which preserves the upper-bound on the norms $\norm{\cdot}_{2,\infty}$ and $\norm{\cdot}_{\infty}$. Consequently, $N \in \mathcal{N}^p_C(l, p \times s)$.

\section{Proof of Proposition \ref{existence}}

The proof amounts to showing that $\L_n$ is continuous on the compact set $\G_n$.

\begin{proof}

Firstly, we note that the map $\L^{ot}_n : T \mapsto \norm{I-T}_{L^2(P_n)}$ is continuous. Secondly, we prove that $\L^{gen}_n : T \mapsto \lambda_n \W_n(T_\sharp P_n, Q_n)$ is Lipschitz continuous. Let $T_1,T_2 \in \C(\Omega,\Omega)$ and compute,

\begin{align*}
    \abs{\W_n({T_1}_\sharp P_n, Q_n) - \W_n({T_2}_\sharp P_n, Q_n)} &\leq \abs{\sup_{D\in \D_n}\left\lbrace\int D(T_1(x))- D(T_2(x))\mathrm{d} P_n(x)\right\rbrace} \\
    &\leq \sup_{D\in \operatorname{Lip}_1(\Omega,\R)}\abs{\int D(T_1(x))- D(T_2(x))\mathrm{d} P_n(x)}\\
    &\leq \int \norm{T_1(x)- T_2(x)}\mathrm{d} P_n(x)\\
    &\leq \norm{T_1-T_2}_{\infty}.
\end{align*}
As a conclusion, $\L_n := \L^{ot}_n + \L^{gen}_n$ is continuous, and as such admits a minimizer on any compact set, in particular $\G_n$.

\end{proof}

\section{Proof of Theorem \ref{main}}

The proof relies on an intermediary result on the minimax estimator described in Section 5 of \citep{hutter2021minimax}. Existence and statistical guarantees follow from the smoothness assumptions \ref{S-source} and \ref{S-map}.
\begin{lemma}\label{minimax}
Assume that Assumptions \ref{S-source} and \ref{S-map} hold, and let ${T^{\operatorname{MM}}_n}$ be the minimax estimator from \citep{hutter2021minimax} of the optimal transport map $T_0$. It satisfies,
\begin{equation}\label{L2}
    \norm{{T^{\operatorname{MM}}_n}-I}^2_{L^2(P_n)} \xrightarrow[n \to +\infty]{a.s.} \norm{T_0-I}^2_{L^2(P)}.
\end{equation}
Additionally, if Assumptions \ref{discriminator}, \ref{lambda} and \ref{epsilon} hold, then
\begin{equation}\label{W1}
    \lambda_n \W_n({T^{\operatorname{MM}}_n}_\sharp P_n,Q_n) \xrightarrow[n \to +\infty]{a.s.} 0,
\end{equation}
hence,
\begin{equation}\label{L}
    \L_n({T^{\operatorname{MM}}_n}) \xrightarrow[n \to +\infty]{a.s.} \norm{T_0-I}^2_{L^2(P)}.
\end{equation}

\end{lemma}

\begin{proof}

Let's start by proving \eqref{L2}. According to the triangle inequality,
\begin{align*}
     \norm{{T^{\operatorname{MM}}_n}-I}_{L^2(P_n)} &\leq  \norm{{T^{\operatorname{MM}}_n}-T_0}_{L^2(P_n)} +  \norm{{T_0}-I}_{L^2(P_n)},\\
     &\leq \sqrt{ \abs{\int \norm{{T^{\operatorname{MM}}_n}-T_0}^2 (\mathrm{d}P_n - \mathrm{d}P)}} + \norm{{T^{\operatorname{MM}}_n}-T_0}_{L^2(P)} +  \norm{{T_0}-I}_{L^2(P_n)}.
\end{align*}
We address each of the three terms of the upper bound in order. For the first term, recall that both $T^{\operatorname{MM}}_n$ and $T$ are $L$-Lipschitz on $\Omega$. Let's show that this entails that $x \mapsto \norm{T^{\operatorname{MM}}_n(x) - T_0(x)}^2$ is Lipschitz. For any $x,y \in \Omega$,
\begin{align*}
    \abs{\norm{T^{\operatorname{MM}}_n(x) - T_0(x)}^2 - \norm{T^{\operatorname{MM}}_n(y) - T_0(y)}^2} &\leq 2 \norm{T^{\operatorname{MM}}_n - T_0}_{\infty} \left(\norm{T^{\operatorname{MM}}_n(x) - T_0(x)} + \norm{T^{\operatorname{MM}}_n(y) - T_0(y)}\right),\\
    &\leq 2 \operatorname{diam}(\Omega)\left(\norm{T^{\operatorname{MM}}_n(x) - T_0(x)} - \norm{T^{\operatorname{MM}}_n(y) - T_0(y)}\right),\\
    &\leq 2 \operatorname{diam}(\Omega)\norm{T^{\operatorname{MM}}_n(x) - T_0(x) - T^{\operatorname{MM}}_n(y) + T_0(y)},\\
    &\leq 2 \operatorname{diam}(\Omega)\left(\norm{T^{\operatorname{MM}}_n(x) - T^{\operatorname{MM}}_n(y)} + \norm{T_0(x) - T_0(y)} \right),\\
    &\leq 2 \operatorname{diam}(\Omega)\left(L \norm{x-y} + L \norm{x-y} \right),\\
    &\leq 4 L \operatorname{diam}(\Omega) \norm{x - y}.
\end{align*}
Denoting $L' = 4 L \operatorname{diam}(\Omega)$, we conclude that $x \mapsto \norm{T^{\operatorname{MM}}_n(x) - T_0(x)}^2$ belongs to $\operatorname{Lip}_{L'}(\Omega,\R)$. As a consequence,
\begin{equation*}
    \abs{\int \norm{{T^{\operatorname{MM}}_n}-T_0}^2 (\mathrm{d}P_n - \mathrm{d}P)} \leq \sup_{f \in \operatorname{Lip}_{L'}(\Omega,\R)} \abs{\int f (\mathrm{d}P_n - \mathrm{d}P)}.
\end{equation*}
The upper bound is a centered empirical process indexed by $\operatorname{Lip}_{L'}(\Omega,\R)$. According to Corollary 2.7.2. and Theorem 2.4.1 in \citep{van1996weak}, it tends to zero almost surely as $n$ increases to infinity. This shows the convergence of the first term.

To control the second term we rely on Proposition 12 in \citep{hutter2021minimax}. It states that with probability at least $1-\delta$,
$$
    \norm{{T^{\operatorname{MM}}_n} - T_0}^2_{L^2(P)} = \begin{cases} O\left( n^{-\frac{4}{2+d}} (\log n)^2 + \frac{\log \delta^{-1}}{n}\right) & \mbox{if } d>2\\ O\left( n^{-1} (\log n)^2 + \frac{\log \delta^{-1}}{n}\right) & \mbox{if } d=2\\ O\left( n^{-1} + \frac{\log \delta^{-1}}{n}\right) & \mbox{if } d=1\\
    \end{cases}
$$
Hence,
\begin{equation}\label{rates}
    \norm{{T^{\operatorname{MM}}_n} - T_0}_{L^2(P)} = \begin{cases} O \left( n^{-\frac{4}{2+d}} (\log n) + \sqrt{\frac{\log \delta^{-1}}{n}} \right) & \mbox{if } d>2\\ O\left(n^{-\frac{1}{2}} (\log n) + \sqrt{\frac{\log \delta^{-1}}{n}}\right) & \mbox{if } d=2\\ O\left( n^{-\frac{1}{2}} + \sqrt{\frac{\log \delta^{-1}}{n}}\right) & \mbox{if } d=1\\
    \end{cases}
\end{equation}
Then, by setting $\delta_n = \frac{1}{n^2}$, it follows from Borel-Cantelli's theorem that $\norm{{T^{\operatorname{MM}}_n} - T_0}_{L^2(P)} \xrightarrow[n \to +\infty]{a.s.} 0$. This shows the desired convergence of the second term. Moreover, as $n$ increases to infinity, the third term of the upper bound tends almost surely to $\norm{{T_0}-I}_{L^2(P)}$, by weak convergence of $P_n$ to $P$ almost surely, 

We now turn to the demonstration of \eqref{W1}. Let $D \in \D_n$ and write the following decomposition,
\begin{align*}
    \int D \circ {T^{\operatorname{MM}}_n} \mathrm{d} P_n - \int D \mathrm{d}Q_n &= \int D \circ {T^{\operatorname{MM}}_n} \mathrm{d}(P_n -P) + \int (D \circ {T^{\operatorname{MM}}_n} - D \circ T_0) \mathrm{d} P + \int D \circ T_0 \mathrm{d}P - \int D \mathrm{d}Q_n,\\
    &\leq \abs{\int D \circ {T^{\operatorname{MM}}_n} \mathrm{d}(P_n -P)} + \int \norm{{T^{\operatorname{MM}}_n} - T_0} \mathrm{d}P + \abs{\int D \mathrm{d}(Q-Q_n)},
\end{align*}
where we use that $\int D \circ T_0 \mathrm{d}P = \int D \mathrm{d}Q$ since ${T_0}_\sharp P = Q$. Noting that $\D_n \subseteq \operatorname{Lip}_1(\Omega,\R)$ we obtain,
\begin{equation*}
    \W_n({T^{\operatorname{MM}}_n}_\sharp P_n,Q_n) \leq \sup_{D\in \operatorname{Lip}_1(\Omega,\R)}\abs{\int D \circ {T^{\operatorname{MM}}_n} \mathrm{d}(P_n -P)} + \int \norm{{T^{\operatorname{MM}}_n} - T_0} \mathrm{d}P + \sup_{D\in \operatorname{Lip}_1(\Omega,\R)}\abs{\int D \mathrm{d}(Q-Q_n)}.
\end{equation*}
Recall now that ${T^{\operatorname{MM}}_n}$ is $L$-Lipschitz so that for any $D \in \D_n$ we have $D \circ {T^{\operatorname{MM}}_n} \in \operatorname{Lip}_L(\Omega,\R)$. As a consequence,
\begin{equation}\label{3terms}
    \W_n({T^{\operatorname{MM}}_n}_\sharp P_n,Q_n) \leq \sup_{g\in \operatorname{Lip}_L(\Omega,\R)}\abs{\int g \mathrm{d}(P_n -P)} + \int \norm{{T^{\operatorname{MM}}_n} - T_0} \mathrm{d}P + \sup_{D\in \operatorname{Lip}_1(\Omega,\R)}\abs{\int D \mathrm{d}(Q-Q_n)}.
\end{equation}
Next, we control each of the three terms of the upper bound in \eqref{3terms} with high probability.

Let us start with the first one, which is the supremum of a centered empirical process indexed by Lipschitz functions. Recall that $P_n$ is supported by $n$ independent variables $x_1,\ldots,x_n \sim P$. Set $X \sim P$ and define
$$
    Z_n := \sup_{g \in \operatorname{Lip}_L(\Omega,\R)} \abs{\frac{1}{n} \sum^n_{i=1} g(x_i) - \E g(X)} = \sup_{g\in \operatorname{Lip}_L(\Omega,\R)}\abs{\int g \mathrm{d}(P_n -P)}.
$$
By $L$-Lipschitz continuity, changing $x_i$ by an independent duplicate $x'_i \sim P$ changes $Z_n$ of at most $\frac{1}{n} L \operatorname{diam}(\Omega)$. Thus, it follows from MacDiarmid's inequality \citep{boucheron2013concentration} that for any $t>0$,
$$
    \P(Z_n \leq \E Z_n + t) \leq 1 - \exp \left( -\frac{2 t^2}{\frac{1}{n}L^2 \operatorname{diam}^2(\Omega)} \right).
$$
After a change of variable, we get for every $0<\delta<1$,
$$
    \P(Z_n \leq \E Z_n + \frac{L\operatorname{diam}(\Omega)}{\sqrt{2n}} \sqrt{\log(\delta^{-1})}) \leq 1 - \delta.
$$
Theorem 4 in \citep{schreuder2020bounding} provides an upper bound on $\E Z_n$. Up to logarithmic factors we have,
$$
    \E Z_n = \begin{cases} O\left( n^{-\frac{1}{d}}\right) & \mbox{if } d > 2\\
    O\left(n^{-\frac{1}{2}} \log n \right) & \mbox{if } d=2\\
    O\left( n^{-\frac{1}{2}}\right) & \mbox{if } d=1\\
    \end{cases}
$$
Hence, with probability at least $1-\delta$,
$$
    Z_n = \begin{cases} O\left( n^{-\frac{1}{d}} + \sqrt{ \frac{\log(\delta^{-1})}{n} }\right) & \mbox{if } d > 2\\
    O\left(n^{-\frac{1}{2}} \log n + \sqrt{ \frac{\log(\delta^{-1})}{n} }\right) & \mbox{if } d=2\\
    O\left(n^{-\frac{1}{2}} + \sqrt{ \frac{\log(\delta^{-1})}{n} }\right) & \mbox{if } d=1\\
    \end{cases}
$$
The third term of \eqref{3terms} can be bounded similarly, as the smoothness $L$ only affects the hidden constant in the $O$. We now turn to the second term of \eqref{3terms}. If follows from Cauchy-Schwarz inequality that
\begin{equation*}
    \int \norm{{T^{\operatorname{MM}}_n} - T_0} \mathrm{d}P \leq \norm{{T^{\operatorname{MM}}_n} - T_0}_{L^2(P)}.
\end{equation*}
Recall that with probability at least $1-\delta$, the right-term of this inequality is bounded as in \eqref{rates}.

By summing the bounds in probability holding for each of the three terms of \eqref{3terms}, and after rescaling $\delta$ by 3, we obtain that with probability at least $1-\delta$,
\begin{equation*}
    \W_n({T^{\operatorname{MM}}_n}_\sharp P_n,Q_n) = \begin{cases} O\left(n^{-\frac{1}{d}} + n^{-\frac{4}{2+d}} (\log n) + \sqrt{\frac{\log \delta^{-1}}{n}}\right) & \mbox{if } d>2\\ O\left(n^{-\frac{1}{2}} (\log n) + \sqrt{\frac{\log \delta^{-1}}{n}}\right) & \mbox{if } d=2\\ O\left(n^{-\frac{1}{2}} + \sqrt{\frac{\log \delta^{-1}}{n}}\right) & \mbox{if } d=1\\
    \end{cases}
\end{equation*}
Now, we replace $\delta$ by $\frac{1}{n^2}$ and we multiply both sides of the inequality by $\lambda_n$ so that with probability at least $1-\frac{1}{n^2}$,
$$
    \lambda_n \W_n({T^{\operatorname{MM}}_n}_\sharp P_n,Q_n) = \begin{cases} \lambda_n O\left( n^{-\frac{1}{d}} + n^{-\frac{4}{2+d}} \log n + \sqrt{ \frac{\log(n)}{n} } \right) & \mbox{if } d > 2\\
    \lambda_n O\left(n^{-\frac{1}{2}} \log n + \sqrt{ \frac{\log(n)}{n} } \right) & \mbox{if } d=2\\
    \lambda_n O\left( n^{-\frac{1}{2}} + \sqrt{ \frac{\log(n)}{n} } \right) & \mbox{if } d=1\\
    \end{cases}
$$
Then, Assumption \ref{lambda} on $\lambda_n$ implies that with probability at least $1-\frac{1}{n^2}$,
\begin{equation*}
    \lambda_n \W_n({T^{\operatorname{MM}}_n}_\sharp P_n,Q_n) = \begin{cases} o(1) + o\left(n^{-\frac{3d-2}{d(2+d)}} \log n\right) + o\left(n^{-\frac{d-2}{2d}} \sqrt{\log(n)}\right) & \mbox{if } d > 2\\
    o(1) + o\left(\frac{1}{\sqrt{\log n}}\right) & \mbox{if } d=2\\
    o\left(\frac{1}{\sqrt{\log(n)}}\right) + o(1) & \mbox{if } d=1\\
    \end{cases}
\end{equation*}
We conclude, using Borel-Cantelli's theorem, that $\lim_{n \to +\infty} \lambda_n \W_n({T^{\operatorname{MM}}_n}_\sharp P_n,Q_n) = 0$ almost surely.

\end{proof}

We now turn to the proof of Theorem \ref{main}, which will be divided in three steps.

\begin{proof}
Recall that for any $n \in \N$, $G_n \in \G_n \subset \G := \operatorname{Lip}_L(\Omega,B_L)$ according to Assumption \ref{generator}. Since $\G$ is a compact set, there exists a subsequence $\{G_{\varphi(n)}\}_{n \in \N}$ and some $G_\varphi \in \G$ such that $\norm{G_{\varphi(n)} - G_\varphi}_{\infty} \xrightarrow[n \to +\infty]{a.s.} 0$. The goal of the proof is to show that $G_\varphi = T_0$ regardless of the extraction $\varphi$. For the sake of clarity, we will not track $\varphi$ in the notations for the rest of the proof.

Moreover, note that since the minimax estimator $T^{\operatorname{MM}}_n$ belongs to $\G$, we know from Assumption \ref{generator} and Theorem \ref{approx_d} that there exists a GroupSort neural network $G^{\operatorname{MM}}_n \in \G_n$ such that $\norm{G^{\operatorname{MM}}_n-T^{\operatorname{MM}}_n}_{\infty} \leq \varepsilon_n$. This neural network approximation of the minimax estimator will play a key role throughout the proof. 

\paragraph{Step 1.} In this first part, we aim at showing that $\lim_{n \to +\infty} \lambda_n \W_n({G_n}_\sharp P_n, Q_n) = 0$ almost surely when $\lambda_n$ verifies Assumption \ref{lambda}. Let's assume ad absurdum that $\lambda_n \W_n({G_n}_\sharp P_n, Q_n)$ does not tend to zero. As $0 \in \D_n$, we have that $\W_n({G_n}_\sharp P_n, Q_n) > 0$ and consequently $\lim_{n \to +\infty} \lambda_n \W_n({G_n}_\sharp P_n, Q_n) = +\infty$. We will show a contradiction to this convergence.

Recall that $\norm{G^{\operatorname{MM}}_n - T^{\operatorname{MM}}_n}_{\infty} \leq \varepsilon_n$, and that $G \mapsto \lambda_n \W_n(G_\sharp P_n, Q_n)$ is $\lambda_n$-Lipschitz continuous. This leads to,
\begin{align*}
    \abs{\L_n(G^{\operatorname{MM}}_n) - \L_n(T^{\operatorname{MM}}_n)} &\leq \lambda_n \abs{ \W_n({G^{\operatorname{MM}}_n}_\sharp P_n,Q_n) - \W_n({T^{\operatorname{MM}}_n}_\sharp P_n,Q_n)} + \norm{I-G^{\operatorname{MM}}_n}_{L^2(P_n)} + \norm{I-T^{\operatorname{MM}}_n}_{L^2(P_n)},\\
    &\leq \lambda_n \norm{G^{\operatorname{MM}}_n - T^{\operatorname{MM}}_n}_{\infty} + \operatorname{diam}^2(\Omega) + \operatorname{diam}^2(\Omega),\\
    &\leq \lambda_n \varepsilon_n + 2 \operatorname{diam}^2(\Omega).
\end{align*}
As $G_n$ minimizes $\L_n$ over $\G_n$, and since $G^{\operatorname{MM}}_n \in \G_n$, we additionally have,
\begin{equation*}
     \L_n(G_n) \leq \L_n(G^{\operatorname{MM}}_n) = \left\{\L_n(G^{\operatorname{MM}}_n) - \L_n(T^{\operatorname{MM}}_n)\right\} + \L_n(T^{\operatorname{MM}}_n).
\end{equation*}
Hence,
\begin{equation*}
    \lambda_n \W_n({G_n}_\sharp P_n, Q_n) + \norm{I-G_n}_{L^2(P_n)} \leq \left\{\lambda_n \varepsilon_n  + 2 \operatorname{diam}^2(\Omega) \right\} + \lambda_n \W_n({T^{\operatorname{MM}}_n}_\sharp P_n, Q_n) + \norm{I-T^{\operatorname{MM}}_n}_{L^2(P_n)},
\end{equation*}
leading to 
\begin{equation*}
    0 \leq \lambda_n \W_n({G_n}_\sharp P_n, Q_n) \leq \lambda_n \varepsilon_n  + 3 \operatorname{diam}^2(\Omega) + \lambda_n \W_n({T^{\operatorname{MM}}_n}_\sharp P_n, Q_n).
\end{equation*}
From Lemma \ref{minimax}, it follows that the right term is bounded, which contradicts $ \lambda_n \W_n({G_n}_\sharp P_n, Q_n) \xrightarrow[n \to +\infty]{a.s.} +\infty$. Consequently, $\W_n({G_n}_\sharp P_n, Q_n) \xrightarrow[n \to +\infty]{a.s.} 0$.

\paragraph{Step 2.} Now, we prove that $G_\sharp P = Q$. Note that,
\begin{align*}
    \abs{\W_n({G_n}_\sharp P_n, Q_n) - W(G_\sharp P, Q)} &\leq \abs{ \sup_{D \in \D_n} \left(\int D \circ G_n \mathrm{d} P_n -  \int D \mathrm{d} Q_n\right) - \left(\int D \circ G \mathrm{d} P - \int D \mathrm{d} Q\right)}\\
    &+ \abs{\sup_{D \in \D_n} \left(\int D \circ G \mathrm{d} P - \int D \mathrm{d} Q\right) - \sup_{D \in \operatorname{Lip}_1(\Omega,\R)} \left(\int D \circ G \mathrm{d} P - \int D \mathrm{d} Q\right)},\\
    &\leq \abs{ \sup_{D \in \operatorname{Lip}_1(\Omega,\R)} \left(\int D \circ G_n \mathrm{d} P_n -  \int D \mathrm{d} Q_n\right) - \left(\int D \circ G \mathrm{d} P - \int D \mathrm{d} Q\right)}\\
    &+ \abs{\sup_{D \in \D_n} \left(\int D \circ G \mathrm{d} P - \int D \mathrm{d} Q\right) - \sup_{D \in \operatorname{Lip}_1(\Omega,\R)} \left(\int D \circ G \mathrm{d} P - \int D \mathrm{d} Q\right)},\\
    &\leq \abs{\sup_{D \in \operatorname{Lip}_1(\Omega,\R)} \int D \circ G_n \mathrm{d} P_n - \int D \circ G \mathrm{d} P} + \abs{\sup_{D \in \operatorname{Lip}_1(\Omega,\R)} \int D (\mathrm{d} P_n - \mathrm{d} P)}\\
    &+ \abs{\sup_{D \in \D_n} \left(\int D \circ G \mathrm{d} P - \int D \mathrm{d} Q\right) - \sup_{D \in \operatorname{Lip}_1(\Omega,\R)} \left(\int D \circ G \mathrm{d} P - \int D \mathrm{d} Q\right)}.
\end{align*}
The second term of the upper bound is the supremum of a centered empirical process indexed by the class of 1-Lipschitz functions, which tends to zero almost surely as $n$ increases to infinity. The third term tends to zero according to Assumption \ref{discriminator}. To address the first term, remark that for any $D \in \operatorname{Lip}_1(\Omega,\R)$,
\begin{equation*}
    D(G_n(x)) \leq \norm{G_n(x) - G(x)} + D(G(x)).
\end{equation*}
Consequently,
\begin{align*}
    \abs{\sup_{D \in \operatorname{Lip}_1(\Omega,\R)} \int D \circ G_n \mathrm{d} P_n - \int D \circ G \mathrm{d} P} &\leq \norm{G_n-G}_{\infty} + \abs{\sup_{D \in \operatorname{Lip}_1(\Omega,\R)} \int (D \circ G) (\mathrm{d} P_n - \mathrm{d} P)},\\
    &\leq \norm{G_n-G}_{\infty} + \abs{\sup_{f \in \operatorname{Lip}_L(\Omega,\R)} \int f (\mathrm{d} P_n - \mathrm{d} P)},
\end{align*}
where we used the fact that $D \circ G \in \operatorname{Lip}_L(\Omega,\R)$, since $D \in \operatorname{Lip}_1(\Omega,\R)$ and $G \in \operatorname{Lip}_L(\Omega,\Omega)$. By definition of $G$, we know that $\norm{G-G_n}_{\infty} \xrightarrow[n \to +\infty]{a.s.} 0$. Moreover, the second term is here again the supremum of a centered empirical process indexed by Lipschitz functions, which tends to zero almost surely. 

All in all, $\W_n({G_n}_\sharp P_n, Q_n) \xrightarrow[n \to +\infty]{a.s.} 0$, and it follows from the first step that $W(G_\sharp P, Q) = 0$, hence $G_\sharp P = Q$.

\paragraph{Step 3.} We know that $G_\sharp P = Q$. To conclude that $G$ is the unique optimal transport map $T_0$ between $P$ and $Q$, we show that $G$ minimizes the transportation cost. Firstly, we write,
\begin{align*}
    \abs{\norm{I-G_n}^2_{L^2(P_n)} - \norm{I-G}^2_{L^2(P)} } &\leq \abs{\norm{I-G_n}^2_{L^2(P_n)} - \norm{I-G}^2_{L^2(P_n)} } + \abs{\norm{I-T_0}^2_{L^2(P_n)} - \norm{I-G}^2_{L^2(P)} },\\
    &\leq 2 \operatorname{diam}(\Omega) \norm{G_n - G}_{\infty} + 2 \operatorname{diam}(\Omega) \abs{ \int \norm{T_0(x)-G(x)}^2 (\mathrm{d}P_n(x) - \mathrm{d}P(x))}.
\end{align*}
Hence,
\begin{equation}\label{lim1}
    \norm{I-G_n}_{L^2(P_n)} \xrightarrow[n \to +\infty]{a.s.} \norm{I-G}_{L^2(P)}.
\end{equation}

Secondly, using that $G_n$ minimizes $\L_n$ on $\G_n$ we have
\begin{align*}
    \L_n(G_n) &\leq \L_n(G^{\operatorname{MM}}_n),\\
              &\leq \lambda_n \sup_{D \in \operatorname{Lip}_1(\Omega,\R)} \left\{\int (D \circ {G^{\operatorname{MM}}_n}) \mathrm{d}P_n - \int D \mathrm{d} Q_n \right\} + \norm{I-G^{\operatorname{MM}}_n}^2_{L^2(P_n)},\\
              &\leq \lambda_n \sup_{D \in \operatorname{Lip}_1(\Omega,\R)} \left\{\int (D \circ {G^{\operatorname{MM}}_n}) \mathrm{d}P_n - \int (D \circ {T^{\operatorname{MM}}_n}) \mathrm{d} P_n \right\} \\ &+ \lambda_n \sup_{D \in \operatorname{Lip}_1(\Omega,\R)} \left\{\int (D \circ {T^{\operatorname{MM}}_n}) \mathrm{d}P_n - \int D \mathrm{d} Q_n \right\}+ \norm{I-G^{\operatorname{MM}}_n}^2_{L^2(P_n)},\\
              &\leq \lambda_n \norm{{T^{\operatorname{MM}}_n} - G^{\operatorname{MM}}_n}_{\infty} + \lambda_n \W_n({T^{\operatorname{MM}}_n}_\sharp P_n, Q_n) + \norm{I-G^{\operatorname{MM}}_n}^2_{L^2(P_n)},\\
              &\leq \lambda_n \varepsilon_n + \L_n({T^{\operatorname{MM}}_n}) + \norm{I-G^{\operatorname{MM}}_n}^2_{L^2(P_n)} - \norm{I-T^{\operatorname{MM}}_n}^2_{L^2(P_n)},\\
              &\leq \lambda_n \varepsilon_n + \L_n({T^{\operatorname{MM}}_n}) + \norm{I-G^{\operatorname{MM}}_n}^2_{L^2(P_n)} - \norm{I-T^{\operatorname{MM}}_n}^2_{L^2(P_n)},\\
              &\leq \lambda_n \varepsilon_n + \L_n({T^{\operatorname{MM}}_n}) + \left(\norm{I-G^{\operatorname{MM}}_n}_{L^2(P_n)} - \norm{I-T^{\operatorname{MM}}_n}_{L^2(P_n)}\right)\\ &\times\left(\norm{I-G^{\operatorname{MM}}_n}_{L^2(P_n)} + \norm{I-T^{\operatorname{MM}}_n}_{L^2(P_n)}\right),\\
              &\leq \lambda_n \varepsilon_n + \L_n({T^{\operatorname{MM}}_n}) + 2 \varepsilon_n \operatorname{diam}(\Omega).
\end{align*}
This inequality can be written as,
\begin{equation*}
    \lambda_n \W_n({G_n}_\sharp P_n,Q_n) + \norm{I-G_n}^2_{L^2(P_n)} \leq \L_n({T^{\operatorname{MM}}_n}) + \lambda_n \varepsilon_n + 2 \varepsilon_n \operatorname{diam}(\Omega).
\end{equation*}
Then, according to the first step of the proof and the convergence \eqref{lim1}, the left term tends almost surely to $\norm{I-G}^2_{L^2(P)}$ as $n$ increases to infinity. Besides, according to Lemma \ref{minimax} and Assumptions \ref{epsilon} and \ref{lambda}, the right term tends to $\norm{I-T_0}^2_{L^2(P)}$. Consequently, 
\begin{equation*}
    \norm{I-G}^2_{L^2(P)} \leq \norm{I-T_0}^2_{L^2(P)}.
\end{equation*}
This means that $G$ minimizes the transportation cost. By uniqueness of the optimal transport map we conclude that $G = T_0$. This completes the proof.

\end{proof}


\end{document}